\documentclass[journal]{IEEEtran}
\IEEEoverridecommandlockouts

\usepackage{cite}
\usepackage{amsmath,amssymb,amsfonts}
\usepackage{graphicx}
\graphicspath{{figs/}}
\usepackage{textcomp}
\usepackage{booktabs}
\usepackage{pifont}
\usepackage{multirow}

\newcommand{\skewm}[1]{[#1]_{\times}}
\newcommand{\pt}{\mathbf{p}}
\newcommand{\nrm}{\mathbf{n}}
\newcommand{\dvec}{\mathbf{d}}
\newcommand{\evec}{\mathbf{e}}
\newcommand{\wvec}{\boldsymbol{\omega}}
\newcommand{\tvec}{\mathbf{t}}
\newcommand{\zerov}{\mathbf{0}}
\newcommand{\R}{\mathbb{R}}
\newcommand{\se}{\mathfrak{se}(3)}
\newcommand{\Ad}{\operatorname{Ad}}
\newcommand{\diag}{\operatorname{diag}}
\newcommand{\yes}{\ding{51}}
\newcommand{\no}{\ding{55}}
\newtheorem{proposition}{Proposition}
\newtheorem{theorem}{Theorem}
\newtheorem{corollary}{Corollary}
\newtheorem{definition}{Definition}

\begin{document}

\title{Degenerate in Whose Frame? An Equivariance Condition\\
for Degeneracy Detection in LiDAR Registration}

\author{Yujie Zhang, Chunlei Zhao, Yuzong Lin, Yuxuan Guo, Xiaohui Jia, and Jinyue Liu%
\thanks{This work has been submitted to the IEEE for possible publication. Copyright may be transferred without notice, after which this version may no longer be accessible.}%
\thanks{All authors are with the School of Mechanical Engineering, Hebei University of Technology, Tianjin 300401, China.}%
\thanks{Corresponding authors: Xiaohui Jia and Jinyue Liu (e-mail: \texttt{xiaohui.jia@placeholder}; \texttt{jinyue.liu@placeholder}).}}

\markboth{IEEE Robotics and Automation Letters. Preprint Version.}%
{Zhang \MakeLowercase{\textit{et al.}}: Degenerate in Whose Frame?}

\maketitle

\begin{abstract}
Degeneracy detectors for LiDAR registration commonly return six per-axis binary labels. We ask whether these labels are properties of the scene. Under a body-frame change, the point-to-plane information matrix transforms by congruence, $H'=\Ad(T)^{\top}H\Ad(T)$, not similarity. Congruence preserves nullity and, through the adjoint reparameterization, identifies the same physical twist subspace; the per-axis footprint and a thresholded spectrum need not be invariant. In a noise-free circular tunnel, shifting the origin by one metre changes which degrees of freedom are flagged. A generalized criterion $Hv=\lambda Mv$ is universally frame-independent over positive-semidefinite information forms if and only if its metric rule is equivariant. No fixed metric qualifies, while a rig-adapted one exists only at zero screw pitch, met in one of nineteen surveyed calibrations. The equivariant point-displacement metric $M=\sum_iJ_i^{\top}J_i$ yields dimensionless, scene-scale-invariant generalized eigenvalues. They are invariant to body frame, consistent changes of length unit and scene scales; the threshold also transfers empirically across sequences. Across 365 frame pairs from four public sequences, labels rarely change at practical extrinsic magnitudes, yet a remapping estimator's correction differs between body-frame choices on $44.5$--$69.5\%$ of pairs, with a median of $0.7$--$4.0$~mm and a maximum of $0.87$~m. The per-axis footprint changes even under the equivariant metric, placing the fundamental issue in the reported quantity.
\end{abstract}


\section{Introduction}

\IEEEPARstart{M}{obile} robots in tunnels, corridors, and underground galleries encounter geometry that repeats along the direction of travel. Point-to-plane ICP then weakly constrains motions tangent to the repeated surfaces, while the optimizer still returns six degrees of freedom and fills unobservable components from the initial guess or noise. Robust LiDAR localization therefore detects weak directions~\cite{Zhang2016,XICP2024,DCReg2026,Papais2026}, usually reporting one binary flag per degree of freedom.

\begin{figure}[t]
\centering
\includegraphics[width=0.95\columnwidth]{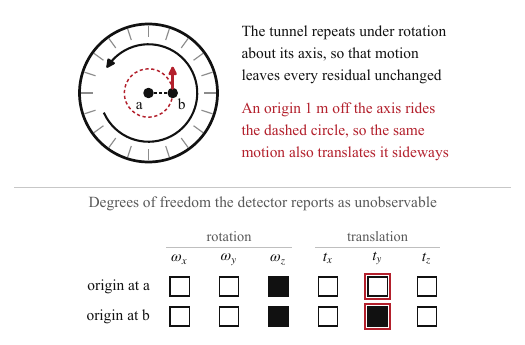}
\caption{Noise-free circular tunnel of radius $2.85$~m with radial normals, in cross-section and to scale. Rotation about its axis leaves every point-to-plane residual unchanged to first order. An on-axis origin represents the twist as pure rotation; shifting the origin one metre off-axis adds a sideways translational coordinate component (dashed circle). A per-axis detector thus reports different unobservable degrees of freedom in the two frames although neither noise nor a threshold enters and the degenerate subspace remains two-dimensional. Table~\ref{tab:instance} repeats the frame change on real data at a published calibration.}
\label{fig:teaser}
\end{figure}

A basic difficulty is the dimensional mismatch. The information matrix
\begin{equation}
H=\sum_i J_i^{\top}\nrm_i\nrm_i^{\top}J_i
\label{eq:H0}
\end{equation}
has a rotational block with units of $\mathrm{m}^2$ relative to its translational block. Its six eigenvalues are therefore not commensurable: changing the length unit can reorder them and alter any raw spectral threshold. Existing detectors normalize $H$ against a reference quadratic form, explicitly or implicitly; this letter asks what transformation that form must obey.

The adjoint maps body perturbations under rigid-body conjugation~\cite{Barfoot2024,Sola2018}, so the Gauss--Newton Hessian changes by congruence. Universal invariance of a generalized spectral criterion requires its normalizing form to transform by the same congruence as $H$; for point-to-plane registration this condition is sufficient. Congruence preserves nullity and adjoint-identifies the same physical twist subspace, but not the per-axis footprint or a soft-threshold decision. Because six independent bits cannot encode a screw, the footprint remains frame-dependent under \emph{any} metric. The reported object, not only its normalization, must therefore change.

\noindent\textbf{(1) Frame-dependent reports.} We separate preserved from frame-dependent quantities and give a noise- and threshold-free circular-tunnel example in which the six per-axis flags change. Across four public sequences the footprint also changes under the proposed criterion, while at practical extrinsic magnitudes a remapping estimator's correction differs on $44.5$--$69.5\%$ of frame pairs, roughly two orders of magnitude more often than the binary label.

\noindent\textbf{(2) Exact universal invariance condition.} Universal frame independence of $Hv=\lambda Mv$ over positive-semidefinite information forms holds exactly when the metric rule is equivariant, $\mathcal{M}(T\!\cdot\!\mathcal{D})=\Ad(T)^{\top}\mathcal{M}(\mathcal{D})\Ad(T)$, excluding the unnormalized, characteristic-length, block-scaled and Jacobi families. A metric adapted to one rig exists only at zero screw pitch, met by one of nineteen surveyed calibrations. The equivariant point-displacement metric $M=\sum_iJ_i^{\top}J_i$ yields dimensionless, scene-scale-invariant eigenvalues.

\noindent\textbf{(3) A dimensionless scale and a portable threshold.} For a constant metric, drift is governed by the effective lever arm $\delta_{\mathrm{eff}}=\|\dvec\|/\ell$, set by the designer through $\ell$. Outside an explicit band no frame change with that lever arm can cross the threshold; none does on $5.97\times10^{5}$ evaluations. Across our transfer tests, only the equivariant criterion preserves every decision under a new frame, a fixed-parameter numerical coordinate rescaling, a scene scale change and a sequence change.

\section{Related Work}

\subsection{Degeneracy Detection and Per-Axis Reports}
Early work related ICP stability to geometry-aware sampling and to the covariance of the linearized normal equations~\cite{Gelfand2003,Censi2007}, and later studies predicted alignment risk before registration and estimated localizability in tunnel-like environments~\cite{Nobili2018,Zhen2019}. Degeneracy-aware estimators eigendecompose the full $6\!\times\!6$ Hessian and remap the update~\cite{Zhang2016}, propagate directional selection into partial SLAM factors~\cite{Hinduja2019}, or separate rotational and translational localizability~\cite{XICP2024}. Recent variants add probabilistic tests~\cite{DRPM2024}, field-scale analysis~\cite{FieldAnalysis2025}, predictive risk~\cite{SuperLoc2025}, Schur-complement decoupling~\cite{DCReg2026}, and block-scaled remapping~\cite{Papais2026}. We compare the implied normalizations while holding the physical observation fixed.

Six binary localizability labels (three translational and three rotational) remain the common output~\cite{Nubert2022}. GEODE broadens the set of real degenerate scenarios but provides scenario-level rather than frame-equivalence annotations~\cite{GEODE2026}. Although convenient, such labels attach physical meaning to selected coordinates of a subspace of $\se$.

\subsection{Coordinate Invariance}
The adjoint relation itself is standard~\cite{Barfoot2024,Sola2018}. Wang \emph{et al.} recently required degeneracy judgments to be invariant to sensor position and coordinate origin, proposed a screw-theoretic formulation, and showed that separating rotation and translation can split one coupled mode into two~\cite{Wang2026}. Our analysis is complementary: it gives a universal necessary and sufficient condition on the normalization rule (Theorem~\ref{thm:iff}), proves that fixed metrics fail except for a measure-zero class of extrinsics, and alters only the coordinate representation of a fixed observation, not the observation itself.

\section{Problem Formulation}

The transformation central to this letter is elementary once the conventions are fixed, but inconsistent perturbation-side and sign conventions can silently introduce errors. We state these conventions before using the transformation.

\subsection{Registration and Conventions}

We consider point-to-plane registration with $N$ correspondences. For the $i$-th pair, let $\pt_i\in\R^3$ be the source point and $\nrm_i\in\R^3$, $\|\nrm_i\|=1$, the unit normal of the target surface, both expressed in the source body frame. We use a right (body) perturbation $\xi=[\wvec^{\top}\ \tvec^{\top}]^{\top}\in\R^6$ with the rotational components first, acting as $T\mapsto T\exp(\xi^{\wedge})$; $\pt_i$ denotes the source point itself rather than its image under the current iterate. The induced point displacement and residual variation are
\begin{align}
\delta\pt_i &= \wvec\times\pt_i+\tvec = J_i\,\xi,
& J_i &= \begin{bmatrix}-\skewm{\pt_i} & I_3\end{bmatrix},
\label{eq:jac}\\
\delta r_i &= \nrm_i^{\top}\delta\pt_i = \mathbf{a}_i^{\top}\xi,
& \mathbf{a}_i &= J_i^{\top}\nrm_i ,
\label{eq:res}
\end{align}
so that $\sum_i(\delta r_i)^2=\xi^{\top}H\xi$ with $H$ as in~\eqref{eq:H0}.

\subsection{Change of Body Frame}

The residuals~\eqref{eq:res} are distances between surfaces and are independent of the choice of origin, so nothing above singles out one body frame over another. Suppose the same registration is instead described about a different body frame (the LiDAR rather than the vehicle body, for example). Let $T=(R,\dvec)\in SE(3)$ denote the pose of the new body frame in the current one, so that $\pt^{\text{old}}=R\,\pt^{\text{new}}+\dvec$. The calibration constant $T$ is unrelated to the pose being estimated. Writing $C:=R^{\top}$, three quantities must be re-expressed: $\tilde{\pt}_i=C(\pt_i-\dvec)$, $\tilde{\nrm}_i=C\nrm_i$ and $\delta\tilde{\pt}_i=C\,\delta\pt_i$. Only the first involves $\dvec$; normals and displacements are free vectors with no point of application, so only point positions are affected by the origin shift. The same physical motion is parameterized by $\xi$ and $\xi'$ in the two frames:
\begin{equation}
\xi=A\,\xi',\qquad
A=\Ad(T)=\begin{bmatrix}R & 0\\ \skewm{\dvec}R & R\end{bmatrix},
\label{eq:Ad}
\end{equation}
which follows from $T\exp(\xi'^{\wedge})T^{-1}=\exp\!\big((\Ad(T)\xi')^{\wedge}\big)$. After reparameterizing its input and re-expressing its output, the Jacobian acquires a factor on each side:
\begin{equation}
\tilde{J}_i=C\,J_i\,A .
\label{eq:Jtilde}
\end{equation}
The extra factor $C$ is easily overlooked; direct expansion confirms~\eqref{eq:Jtilde}, whose blocks are $-C(\skewm{\pt_i}-\skewm{\dvec})C^{\top}=-\skewm{\tilde\pt_i}$ and $CR=I_3$. The residual gradient transforms with a single factor, $\tilde{\mathbf{a}}_i=\tilde{J}_i^{\top}\tilde{\nrm}_i=(CJ_iA)^{\top}(C\nrm_i)=A^{\top}\mathbf{a}_i$, because $C^{\top}C=I$ and a point-to-plane distance is a scalar that cannot depend on the frame in which it is measured. Summing outer products gives
\begin{equation}
\boxed{\;H'=A^{\top}HA\;}
\label{eq:congruence}
\end{equation}
This is a congruence transformation. Since $\det A=1$ and, as shown below, $A$ is orthogonal only when $\dvec=\zerov$, it preserves inertia but not the spectrum. Two structural facts follow. Writing
\begin{equation}
A^{\top}A=\begin{bmatrix}
I+R^{\top}(\|\dvec\|^2I-\dvec\dvec^{\top})R & -R^{\top}\skewm{\dvec}R\\
R^{\top}\skewm{\dvec}R & I
\end{bmatrix},
\label{eq:AtA}
\end{equation}
and noting that $\|\dvec\|^2I-\dvec\dvec^{\top}$ vanishes only for $\dvec=\zerov$ gives the following.

\begin{proposition}[Orthogonality]
\label{prop:ortho}
$\Ad(T)$ is orthogonal if and only if $\dvec=\zerov$. Consequently, the map $H\mapsto A^{\top}HA$ preserves the ordinary spectrum of every symmetric $H$ if and only if the frame change is a pure rotation.
\end{proposition}

The rotational part of a frame change therefore does not affect the spectrum; the entire effect arises from the translational part. Second, factoring $A=DQ$ with $Q=\diag(R,R)$ orthogonal and $D$ the corresponding purely translational adjoint, the spectrum of $A^{\top}A$ is independent of $R$ and admits the closed form
\begin{equation}
\sigma_{\max}(A)=\tfrac{\delta}{2}+\sqrt{1+\tfrac{\delta^2}{4}},\qquad
\sigma_{\min}(A)=\sigma_{\max}(A)^{-1},
\label{eq:sigma}
\end{equation}
with $\delta=\|\dvec\|$, the remaining two singular values equalling unity. These use the raw mixed rotation--translation coordinate norm and depend on the chosen length unit; they are coordinate diagnostics, not unit-invariant physical quantities. Section~\ref{sec:certificate} returns to~\eqref{eq:sigma} and replaces this scale by the dimensionless $\delta_{\mathrm{eff}}$.

\subsection{What Is Preserved, and What Is Not}
\label{sec:tiers}

A congruence such as~\eqref{eq:congruence} is not a similarity, and a detector reports more than a single invariant. Since $A$ is invertible, Sylvester's law gives $\operatorname{rank}H'=\operatorname{rank}H$ and $\mathcal{N}(H')=A^{-1}\mathcal{N}(H)$. Four quantities must be distinguished: (1) degenerate-subspace \emph{dimension} and (2) the physical subspace \emph{in $\se$}, adjoint-identified across frames despite different coordinates, are preserved; (3) its \emph{per-axis coordinate footprint}, which of the six basis twists it involves, and (4) a \emph{soft-threshold} outcome are not. The first two are preserved; the third and fourth are detector reports. The following exact instance separates them without noise or threshold placement.

Consider a circular tunnel of radius $2.85$~m with exactly radial normals and the body origin on its axis (Fig.~\ref{fig:teaser}). Translation along the axis and rotation about it leave every residual unchanged to first order, so the degenerate subspace is two-dimensional and contains a twist with zero translational part. Moving the origin one metre perpendicular to the axis, without rotation or noise, preserves rank and subspace dimension but changes the coordinate representation from $[0\ 0\ 1\ 0\ 0\ 0]^{\top}$ to $[0\ 0\ 1\ 0\ 1\ 0]^{\top}$, an angle of $45^{\circ}$ between the two tuples in the Euclidean norm of the displayed metre-based coordinates. This angle is only a coordinate diagnostic. The new-frame subspace contains \emph{no} twist with zero translational part, so a per-axis detector answers ``Is one of the degenerate directions a pure rotation?'' differently for the same physical situation in the two frames.

\section{Invariance: A Necessary and Sufficient Condition}
\label{sec:invariance}

That instance is not a pathology but a direct consequence of congruence. It raises the central question: which spectral criteria, if any, remain invariant under a change of frame? The answer constrains not a matrix but the rule that produces one.

\subsection{Equivariant Metric Rules}
\label{sec:equivariance}

Since the spectrum of $H$ alone is not commensurable, degeneracy is assessed through a generalized eigenvalue problem
\begin{equation}
H\,v=\lambda\,M\,v ,
\label{eq:gep}
\end{equation}
where $M$ is a metric on the twist space. Direct eigendecomposition of $H$ corresponds to $M=I_6$; normalization by a characteristic length $\ell$ to $M=\diag(\ell^2I_3,I_3)$; the block scaling of~\cite{Papais2026} to the same family with $\ell$ recomputed from $H$; a Jacobi preconditioner to $M=\diag(H)$. What follows concerns the \emph{rule} that assigns a metric to a data set, not a matrix.

\begin{definition}[Equivariance]
A metric rule $\mathcal{M}$ is \emph{equivariant} if for every $T\in SE(3)$ and every data set $\mathcal{D}$,
\begin{equation}
\mathcal{M}(T\!\cdot\!\mathcal{D})=\Ad(T)^{\top}\,\mathcal{M}(\mathcal{D})\,\Ad(T).
\label{eq:equivariance}
\end{equation}
\end{definition}

\begin{theorem}
\label{thm:iff}
Let $M,M'\succ0$. The generalized eigenvalues of $(H,M)$ and $(A^{\top}HA,M')$ agree for every symmetric positive-semidefinite $H$ if and only if $M'=A^{\top}MA$; if agreement up to a common positive factor $c$ is required, then $M'=c^{-1}A^{\top}MA$. Consequently, universal frame independence of the criterion~\eqref{eq:gep} over positive-semidefinite information forms is equivalent to equivariance of $\mathcal{M}$; for point-to-plane Hessians, equivariance is therefore sufficient.
\end{theorem}

\begin{IEEEproof}
Sufficiency: substituting $H'=A^{\top}HA$ and $M'=A^{\top}MA$ into $H'v'=\lambda'M'v'$ and left-multiplying by $A^{-\top}$ gives $H(Av')=\lambda'M(Av')$, so $\lambda'\in\lambda(H,M)$ with eigenvector $v=Av'$. Necessity: with $u=Av'$ the problem reads $Hu=\lambda'\hat{M}u$, where $\hat{M}:=A^{-\top}M'A^{-1}\succ0$. Take $H=M$; since $\lambda(M,M)=(1,\dots,1)$, agreement forces $\lambda(M,\hat{M})=(1,\dots,1)$. The pencil is symmetric-definite and admits a simultaneous congruence diagonalization $X^{\top}\hat{M}X=I$, $X^{\top}MX=\Lambda$ with $X$ invertible; all eigenvalues equal to unity give $\Lambda=I$, hence $M=\hat{M}$, i.e.\ $M'=A^{\top}MA$. The scaled version follows identically. Within the point-to-plane model, repeating each point with three mutually orthogonal normals yields $\sum_kJ^{\top}\nrm_k\nrm_k^{\top}J=J^{\top}J$; this realizes the point-displacement form but is not needed for the universal necessity argument above.
\end{IEEEproof}

For the point-displacement metric introduced below, the condition $M\succ0$ holds unless all points are collinear; coplanar configurations are admissible.

For metrics of the weighted form $M=\sum_iJ_i^{\top}W_iJ_i$ with $W_i$ symmetric, equivariance is equivalent to $W_i'=CW_iC^{\top}$. Comparing $\mathcal{D}$ with $\mathcal{D}\cup\{j\}$ isolates $J_j^{\top}\Delta_jJ_j=0$ with $\Delta_j=C^{\top}W_j'C-W_j$. Since $J_j$ has full row rank, $\Delta_j=0$. The weight must therefore be a covariant tensor on the displacement space rather than a fixed array of numbers.

\subsection{No Fixed Metric Is Equivariant}

Most detectors use a fixed matrix rather than a rule. Such a matrix satisfies~\eqref{eq:equivariance} only if it is invariant under the adjoint. Two propositions settle when this occurs; the first covers the isotropic case, including both the unnormalized spectrum and characteristic-length scaling.

\begin{proposition}[Isotropic constant metrics]
\label{prop:isotropic}
For $M_0=\diag(\alpha I_3,\beta I_3)$ with $\alpha,\beta>0$, the matrix $A^{\top}M_0A$ has the block structure of~\eqref{eq:AtA} with diagonal blocks $\alpha I+\beta R^{\top}(\|\dvec\|^2I-\dvec\dvec^{\top})R$ and $\beta I$, and off-diagonal block $-\beta R^{\top}\skewm{\dvec}R$, which vanishes if and only if $\dvec=\zerov$, independently of $R$, $\alpha$ and $\beta$.
\end{proposition}

The second proposition addresses a weaker requirement: not a metric valid for all frame changes, but one adapted to the single extrinsic of a particular rig.

\begin{proposition}[Adapted constant metrics]
\label{prop:adapted}
Given $T=(R,\dvec)$ with rotation angle $\varphi\neq0$ about the unit axis $\evec$, there exists a constant $M_0\succ0$ with $\Ad(T)^{\top}M_0\Ad(T)=M_0$ if and only if $\evec^{\top}\dvec=0$. Allowing a common positive factor does not weaken the condition, since taking determinants and using $\det\Ad(T)=1$ forces that factor to unity. A pure translation admits no such metric.
\end{proposition}

\begin{IEEEproof}[Sketch]
By Chasles' theorem the screw translation along the axis is $h=\evec^{\top}\dvec$; $h=0$ means $T$ is a rotation about an axis in space, so $T=P R_0 P^{-1}$ for a pure translation $P$ carrying that axis to the origin, and $M_0=\Ad(P)^{-\top}\Ad(P)^{-1}$ is invariant because $\Ad(R_0)$ is orthogonal. Conversely, an invariant positive definite form makes $\{\Ad(T)^n\}$ bounded; the eigenvalues of $\Ad(T)$ all have unit modulus, so boundedness requires diagonalizability. In canonical screw coordinates, $h\neq0$ creates a nontrivial Jordan coupling between the rotational and translational eigenspaces. Away from the eigenvalue-coalescence case, this appears as geometric multiplicity one versus algebraic multiplicity two; at $\varphi=\pi$, the same coupling still leaves $\Ad(T)$ defective. A pure translation gives $\Ad(T)=I+\text{nilpotent}$.
\end{IEEEproof}

The condition $\evec^{\top}\dvec=0$ defines a measure-zero set. Section~\ref{sec:extrinsics} reports how often it is met in practice.

Data-derived rules also fail when they do not transform by congruence. A Jacobi preconditioner satisfies $\diag(A^{\top}HA)\neq A^{\top}\diag(H)A$ in general: it responds to a frame change, but not as~\eqref{eq:equivariance} requires. Table~\ref{tab:metrics} classifies the families accordingly. The two failure modes are distinct and, as Section~\ref{sec:transfer} shows, experimentally distinguishable.

\begin{table}[t]
\caption{Metric Families Under the Equivariance Condition}
\label{tab:metrics}
\centering
\footnotesize
\begin{tabular}{@{}llc@{}}
\toprule
Family & Metric rule & Equivariant \\
\midrule
Unnormalized spectrum~\cite{Zhang2016} & $M=I_6$, fixed & \no \\
Characteristic length & $M=\diag(\ell^2I,I)$, fixed & \no \\
Block scaling~\cite{Papais2026} & same family, $\ell$ from $H$ & \no \\
Jacobi preconditioning & $M=\diag(H)$ & \no \\
\textbf{Point displacement} & $M=\sum_iJ_i^{\top}J_i$ & \yes \\
\bottomrule
\end{tabular}
\end{table}

\subsection{The Point-Displacement Metric}

The weighted form at the end of Section~\ref{sec:equivariance} shows that an admissible rule requires a tensorial weight on the displacement space. The simplest is the identity: $W_i=I_3$ gives
\begin{equation}
M=\sum_{i=1}^{N}J_i^{\top}J_i,\qquad \xi^{\top}M\xi=\sum_i\|\delta\pt_i\|^2 ,
\label{eq:M}
\end{equation}
the total squared surface displacement induced by $\xi$. Since $CI_3C^{\top}=I_3$, the rule is equivariant. The physical interpretation is direct: displacement energy is a scalar and cannot depend on the frame in which it is computed.

Two further properties follow from the same congruence structure. A uniform scaling of the point cloud, $\pt\mapsto s\pt$, yields $J(s\pt)=J(\pt)\,S$ with $S=\diag(sI_3,I_3)$, hence $H\mapsto S^{\top}HS$ and $M\mapsto S^{\top}MS$. A change of length unit is the same map up to a common positive factor; both are congruences applied to $H$ and $M$ alike.

\begin{corollary}
\label{cor:invariance}
The generalized eigenvalues $\lambda(H,M)$ are simultaneously invariant under the choice of body frame, under a change of length unit, and under a uniform scaling of the scene. Moreover $\lambda\in[0,1]$, since $\|\nrm_i\|=1$ implies $(\nrm_i^{\top}\mathbf{x})^2\le\|\mathbf{x}\|^2$ termwise.
\end{corollary}

The bound permits $\lambda$ to be read as the fraction of induced surface displacement that the point-to-plane residual observes, so that an absolute threshold $\tau$ carries meaning across scenes and sensors, provided $H$ and $M$ are assembled from the same correspondences and weights.

\subsection{A Dimensionless Scale and a Safety Certificate}
\label{sec:certificate}

For fixed metrics, the magnitude of frame dependence is partly designer-set. Characteristic-length scaling is invariant to a pure unit change when the dimensioned $\ell$ is converted with that unit, but it remains frame-dependent. Rescaling coordinates while holding the numerical $\ell$ fixed instead changes the physical normalization. For $M_0=\diag(\ell^2I,I)$ this dependence is governed by $\tilde{B}=M_0^{1/2}AM_0^{-1/2}$, yielding
\begin{equation}
\tilde{B}=\Ad\big((R,\ \dvec/\ell)\big),\qquad
\delta_{\mathrm{eff}}=\|\dvec\|/\ell ,
\label{eq:deff}
\end{equation}
The effective lever arm $\delta_{\mathrm{eff}}$ is dimensionless. The unnormalized spectrum is $\ell=1$ in the current coordinate convention, so a unit change alters its implicit physical characteristic length. Increasing $\ell$ attenuates, but does not remove, frame dependence.

Combining~\eqref{eq:deff} with Ostrowski's theorem~\cite{Ostrowski1959} bounds the drift of each eigenvalue by $\theta_k\in[\sigma_{\min}^2,\sigma_{\max}^2]$ evaluated at $\delta_{\mathrm{eff}}$. The resulting band is conservative and is not used to predict drift; instead, it certifies when a threshold crossing is impossible.

\begin{corollary}[Certificate]
\label{cor:cert}
If $\lambda_k/\kappa\notin[\sigma_{\min}^2,\sigma_{\max}^2]$ evaluated at $\delta_{\mathrm{eff}}$, then no change of body frame with that lever arm can move $\lambda_k$ across that detector's own threshold $\kappa$.
\end{corollary}

\section{Experiments}
\label{sec:experiments}

The analysis leaves three empirical questions: the practical magnitude of frame effects; whether the four quantities of Section~\ref{sec:tiers} behave as predicted; and their downstream effect. Because the comparison re-expresses one information matrix in two frames, implementation checks come first.

\subsection{Design}

We evaluate on four public sequences recorded in structurally repetitive environments: GEODE~\cite{GEODE2026} (shield tunnel, Livox Avia), the NTNU LiDAR-degeneracy data~\cite{Nissov2024} (bicycle tunnel, Ouster OS0-128), RTS-GT~\cite{Vaidis2024} (campus underground tunnel, 16-beam Velodyne), and M3DGR~\cite{ZhangM3DGR2025} (indoor corridor, Livox Mid-360), giving 365 registered frame pairs in total. Point clouds are downsampled to a $0.10$~m voxel grid, normals are estimated with a fixed physical radius of $0.30$~m and a minimum neighbourhood of eight points, correspondences are gated at $1.0$~m, and the residual weighting is the symmetrized dual-normal point-to-plane form.

For each pair the point cloud, correspondences and physical motion are fixed; only the body frame changes. We use five arms: the identity; pure rotation $(R,\zerov)$; pure translation $(I,\dvec)$; general $(R,\dvec)$; and an orthogonal control, where $A$ is a random orthogonal $6\times6$ non-adjoint. The last three are sampled at fourteen magnitudes $\|\dvec\|$ from $0$ to $20$~m with twenty random draws apiece; the pure rotation is drawn twenty times and the identity once. This gives $861$ frame changes per pair, $3.14\times10^{5}$ total, and $2.51\times10^{6}$ evaluations over eight metric configurations: the seven of Table~\ref{tab:flips} plus characteristic-length scaling at $\ell=1$, equal to $I_6$ in metres. Unless stated otherwise, lever-arm summaries use the pure-translation arm; the general arm contributes only to the full evaluation count.

Two checks precede the analysis. Rebuilding $H$ and $M$ from the exported correspondences reproduces the pipeline's generalized eigenvalues to a largest relative deviation of $3.7\times10^{-12}$ over all 365 pairs, and applying $A$ to $H$ and $M$ agrees with re-expressing the point cloud and reassembling from scratch to $1.6\times10^{-15}$ and $2.4\times10^{-15}$ respectively. Both checks are sensitive: omitting the factor $C$ in~\eqref{eq:Jtilde}, exchanging $A$ and $A^{-1}$, leaving normals in the target frame, and assembling $H$ and $M$ from different point sets are each detected by one of them.

\subsection{Thresholds Anchored to Analytic Answers}
\label{sec:kappa}

Each detector needs a threshold, and calibrating a baseline against the proposed criterion would presuppose the conclusion. We therefore use synthetic corridors whose number of degenerate directions follows from geometry: a rectangular straight corridor has one (translation along the axis), a circular one has two (translation along the axis and rotation about it), and a right-angle corner has none. The synthetic pipeline follows the real one in voxel downsampling, fixed-radius normal estimation \emph{from noisy points}, and weighting; only the geometry differs. Three range-noise levels and three voxel sizes define the nine preprocessing configurations; crossed with eight seeds over the three corridors, they give 216 instances. To prevent an alignment artefact, the voxel grid origin is randomized because the corridor walls lie at coordinates commensurate with the three voxel sizes: a fixed grid would align the source cloud but not the rotated target cloud. At the coarsest voxel, the ratio of source to target points carrying a valid normal would then be $1.63$, against $0.97$ once the origin is randomized. For each detector, we identify the threshold interval that reproduces the counts $1$, $2$ and $0$ simultaneously; choosing its best threshold deliberately favours the baselines.

The proposed criterion admits such an interval in all nine configurations, and their intersection spans $[5.3\times10^{-3},\,8.3\times10^{-2}]$, a width of $1.20$ decades containing the value $\tau=10^{-2}$ used throughout. For every other metric the intersection is empty (Fig.~\ref{fig:kappa}), and the way each fails differs. Characteristic-length scaling at $\ell=0.25$ and the Jacobi preconditioner admit no threshold reproducing the three counts within \emph{any} single configuration, reproducing only two and one of them respectively. The unnormalized spectrum, which in metres coincides with $\ell=1$, and $\ell=10$ each fail in one configuration and in the same one: the coarsest voxel at the highest noise level, where fixed-radius normal estimation is already at its limit. Block scaling and $\ell=2.85$ admit an interval in every configuration taken alone yet share no value across the nine, the gaps being $0.79$ and $0.13$ decades; the second is narrow enough that an aligned voxel grid closes it, which is why the randomization above matters. The key point is therefore not whether a baseline can be calibrated, but that calibrating one also fixes the preprocessing. Where the intersection is empty, the threshold is taken from the configuration closest to the real data; where no interval reproducing all three exists anywhere, it is the geometric midpoint of the widest interval reproducing the largest number of them, in that same configuration. Both rules were fixed before the runs, and in either case the failure is reported rather than the metric recalibrated against the proposed criterion. All thresholds below are those of this section, except that the proposed criterion keeps the round $\tau=10^{-2}$ rather than the geometric midpoint $2.1\times10^{-2}$ of its interval.

\begin{figure}[t]
\centering
\includegraphics[width=\columnwidth]{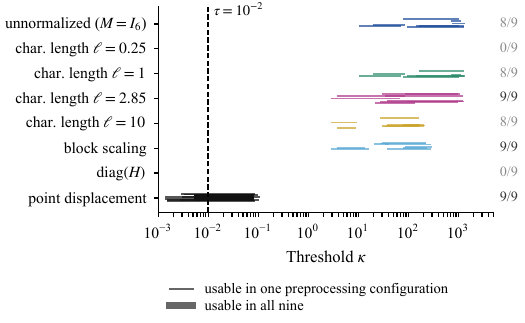}
\caption{Threshold intervals that reproduce the analytically known degenerate-direction counts of the three synthetic corridors across the nine preprocessing configurations. A metric is usable across configurations only where its intervals overlap; this occurs only for the point-displacement metric.}
\label{fig:kappa}
\end{figure}

\subsection{Decision Stability}

With every threshold calibrated to known-in-advance answers, the frame changes can be applied. Table~\ref{tab:flips} and Fig.~\ref{fig:flips} report how often a detector changes its degeneracy decision when only the body frame changes, with Wilson $95\%$ confidence intervals from $7300$ samples per cell. Three observations emerge.

First, for isotropic metrics only translation matters: pure rotation leaves them unchanged, as Proposition~\ref{prop:ortho} requires. The Jacobi preconditioner is the exception, changing $7.6\%$ of decisions because $\diag(H)$ is axis-defined. Second, a random orthogonal replacement for the adjoint leaves the unnormalized spectrum unchanged, isolating the non-orthogonality of $\Ad(T)$. Constant metrics $\diag(\alpha I_3,\beta I_3)$ with $\alpha\neq\beta$ (Proposition~\ref{prop:isotropic}) respond at rates between $8.0\%$ and $40.3\%$ because they are not invariant to a general orthogonal map. Third, the proposed criterion is unchanged in every arm and at every magnitude.

These rates should be read at practical lever arms, measured in Section~\ref{sec:extrinsics} at $0.032$ to $0.552$~m over nineteen real extrinsics. At $\|\dvec\|=0.1$~m, near that median, the label flip rate is at most $0.2\%$ for the fixed and block-scaled metrics; at $0.5$~m, the upper end, at most $1.0\%$. Binary labels are fairly robust at practical lever arms. The downstream correction is not, as Section~\ref{sec:downstream} shows.

\begin{table}[t]
\caption{Decision Changes Across the Pure-Rotation, Pure-Translation, and Orthogonal-Control Arms}
\label{tab:flips}
\centering
\footnotesize
\setlength{\tabcolsep}{5.5pt}
\begin{tabular}{@{}lccccc@{}}
\toprule
\multirow{2}{*}{Metric} & \multicolumn{4}{c}{$\Ad(T)$ arms} & Orthogonal \\
\cmidrule(lr){2-5}
 & rot.\ only & $0.1$~m & $0.5$~m & $20$~m & control \\
\midrule
$I_6$                       & 0.000 & 0.001 & 0.010 & 0.204 & 0.000 \\
$\diag(\ell^2I,I)$, $\ell{=}0.25$ & 0.000 & 0.001 & 0.003 & 0.452 & 0.080 \\
$\diag(\ell^2I,I)$, $\ell{=}2.85$ & 0.000 & 0.001 & 0.008 & 0.393 & 0.403 \\
$\diag(\ell^2I,I)$, $\ell{=}10$   & 0.000 & 0.002 & 0.001 & 0.208 & 0.361 \\
Block scaling~\cite{Papais2026}   & 0.000 & 0.001 & 0.006 & 0.440 & 0.113 \\
$\diag(H)$                        & \textbf{0.076} & 0.083 & 0.081 & 0.882 & 0.643 \\
\textbf{Point displacement}       & \textbf{0.000} & \textbf{0.000} & \textbf{0.000} & \textbf{0.000} & \textbf{0.000} \\
\bottomrule
\end{tabular}
\end{table}

\begin{figure}[t]
\centering
\includegraphics[width=\columnwidth]{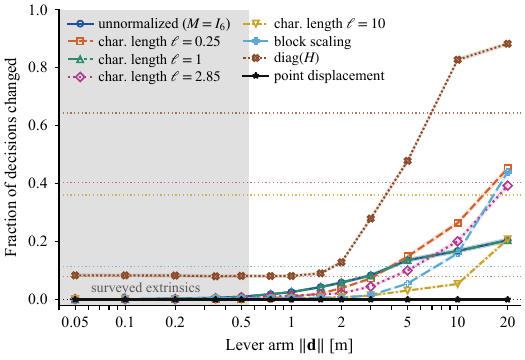}
\caption{Decision-change rate versus pure-translation lever arm, with Wilson $95\%$ confidence bands. The shaded strip spans the nineteen surveyed lever arms; dotted lines show the orthogonal control, where $\Ad(T)$ is replaced by a random orthogonal matrix.}
\label{fig:flips}
\end{figure}

\subsection{What the Two Parts of a Frame Change Break}
\label{sec:dissoc}

A binary flag is the coarsest report and the only one measured so far, so the rates above appear reassuring; Table~\ref{tab:dissoc} resolves the same experiment by reported quantity. Under either metric, pure rotation leaves the binary decision, degenerate-direction count, and rotational-versus-translational character of the weakest direction unchanged while altering the per-axis footprint on $87\%$ of pairs: axes are renamed without mixing the two blocks. Translation alters all four under the unnormalized spectrum, the empirical counterpart of Proposition~\ref{prop:ortho}.

The last row of each block is the one that merits particular attention. Under a pure rotation and under the orthogonal control the per-axis footprint changes at essentially the same rate for the equivariant metric as for the unnormalized spectrum, $87\%$ and up to $99.7\%$ respectively. This is not a shortcoming of a particular normalization: a one-dimensional degenerate mode is a screw, six independent per-axis bits cannot represent one, and no choice of $M$ alters that. The invariance established in Section~\ref{sec:invariance} concerns the eigenvalues and the subspace; we therefore propose reporting the latter as a subspace rather than as axes.

\begin{table}[t]
\caption{Change Rates by Reported Quantity}
\label{tab:dissoc}
\centering
\footnotesize
\begin{tabular}{@{}llcccc@{}}
\toprule
Metric & Quantity & rot.\ only & $\|\dvec\|{=}1$ & $\|\dvec\|{=}20$ & orth.\ ctrl. \\
\midrule
\multirow{4}{*}{$I_6$}
 & binary flag       & 0.000 & 0.024 & 0.204 & 0.000 \\
 & no.\ of directions& 0.000 & 0.142 & 0.851 & 0.000 \\
 & rot./transl.\ type& 0.000 & 0.037 & 0.961 & 0.687 \\
 & per-axis footprint& \textbf{0.875} & 0.125 & 0.978 & 0.998 \\
\midrule
\multirow{4}{*}{\shortstack[l]{Point\\displ.}}
 & binary flag       & 0.000 & 0.000 & 0.000 & 0.000 \\
 & no.\ of directions& 0.000 & 0.000 & 0.000 & 0.000 \\
 & rot./transl.\ type& 0.000 & 0.002 & 0.064 & 0.543 \\
 & per-axis footprint& \textbf{0.872} & 0.063 & 0.520 & 0.997 \\
\bottomrule
\end{tabular}
\end{table}

\subsection{Threshold Transfer}
\label{sec:transfer}

Frame stability is one property; threshold portability is another, often the one encountered first in integration. Each threshold is held fixed and compared with its pre-transfer decision. Four transfers alter one quantity each: body frame, numerical coordinate scale with detector parameters fixed, uniform scene scale, and sequence. The first complements Table~\ref{tab:flips}; the other three appear in Table~\ref{tab:transfer}. A consistent physical unit conversion also converts any dimensioned characteristic length, unlike the fixed-parameter coordinate-rescaling stress test here.

The two failure modes in Table~\ref{tab:metrics} separate cleanly. $H$-derived normalizations (block scaling and Jacobi preconditioning) remain stable under fixed-parameter coordinate and scene rescaling but fail under body-frame or sequence changes, whose worst four-by-four transfer cells are $0.06$ and $0.00$ respectively. Numerically fixed metrics show the opposite coordinate-rescaling behavior: the worst tabulated cell retains $0.55$, and one unshown characteristic length falls to $0.49$. A characteristic-length detector with dimensioned $\ell$ converted consistently should not fail a pure unit conversion. The proposed criterion gives unity agreement in all four tested transfers without retuning a dimensioned normalization parameter. Corollary~\ref{cor:invariance} guarantees the first three for the point-displacement rule; the fourth holds empirically across all sixteen comparison cells.

\begin{table}[t]
\caption{Decision Agreement After Transporting the Threshold}
\label{tab:transfer}
\centering
\footnotesize
\begin{tabular}{@{}lcccc@{}}
\toprule
\multirow{2}{*}{Metric} & \multicolumn{2}{c}{Coord. rescale} & Scene & Cross-seq. \\
\cmidrule(lr){2-3}
 & $\times10^{-2}$ & $\times10^{3}$ & $\times0.1$ & worst cell \\
\midrule
$I_6$                       & 0.795 & 0.997 & 0.841 & 0.169 \\
$\diag(\ell^2I,I)$, $\ell{=}2.85$ & 0.553 & 0.940 & 0.553 & 0.068 \\
Block scaling~\cite{Papais2026}   & 1.000 & 1.000 & 1.000 & 0.059 \\
$\diag(H)$                        & 1.000 & 1.000 & 1.000 & 0.000 \\
\textbf{Point displacement}       & \textbf{1.000} & \textbf{1.000} & \textbf{1.000} & \textbf{1.000} \\
\bottomrule
\end{tabular}
\end{table}

\subsection{Consequence for the Applied Correction}
\label{sec:downstream}

Angles below are Euclidean coordinate-space diagnostics in the metre-based convention, not unit-invariant physical angles. A detector projects the pose update onto the well-conditioned subspace; another sensor supplies the discarded component, or it remains unchanged. With $P=\sum_{k\notin\mathcal{Z}}v_kv_k^{\top}M$, where $\mathcal{Z}$ indexes declared-degenerate directions and the $v_k$ are $M$-orthonormal, two frame choices differ by $\Delta=(P_1-A\,P_2\,A^{-1})\,\xi_{\mathrm{reg}}$, which vanishes for an equivariant metric. Reversing the congruence or using a Euclidean projector in place of the $M$-orthogonal one raises the discrepancy from $10^{-16}$~m to $9.4\times10^{-3}$ and $2.8\times10^{-7}$~m respectively.

For the unnormalized spectrum in the pure-translation arm, $\|\dvec\|=0.5$~m gives a median translational difference of $4.0$~mm, a $95$th percentile of $12$~cm, a maximum of $0.87$~m, and $69.5\%$ of pairs above one millimetre (Wilson $95\%$: $[0.684,0.705]$). At $0.1$~m the corresponding values are $0.7$~mm, $2.1$~cm, $0.87$~m, and $44.5\%$; the other baselines are comparable.

\begin{table}[t]
\caption{The Same Registration in Two Body Frames, for One Frame Pair at That Rig's Published Calibration}
\label{tab:instance}
\centering
\footnotesize
\setlength{\tabcolsep}{5pt}
\begin{tabular}{@{}lcc@{}}
\toprule
Reported quantity & Unnormalized & Point displacement \\
\midrule
binary degeneracy flag          & unchanged & unchanged \\
no.\ of degenerate directions   & unchanged & unchanged \\
degenerate subspace             & $0.22^{\circ}$ & $0.000^{\circ}$ \\
smallest eigenvalue             & $-1.4\%$ & $2\times10^{-15}$ \\
six per-axis flags              & \textbf{changes} & \textbf{changes} \\
applied pose correction         & $6.3$~mm & $6\times10^{-16}$~m \\
\bottomrule
\end{tabular}
\end{table}

\begin{figure}[t]
\centering
\includegraphics[width=\columnwidth]{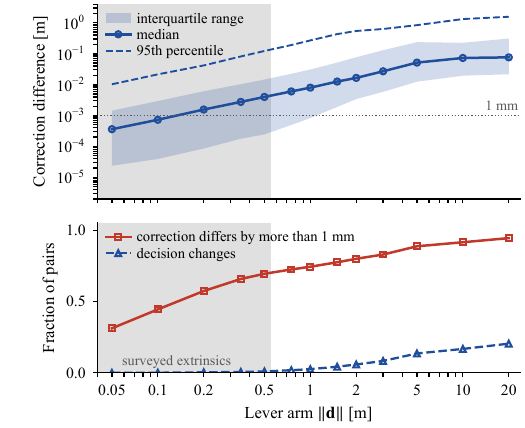}
\caption{Translational correction difference versus lever arm for the unnormalized spectrum. Boxes show quartiles, medians, and fifth--ninety-fifth percentiles; the solid curve is the fraction above one millimetre and the dashed curve is the label-flip rate. Downstream changes are roughly two orders of magnitude more frequent at practical lever arms. The point-displacement result is identically zero.}
\label{fig:downstream}
\end{figure}

At $0.1$~m the label flips on only $0.1\%$ of pairs, against $44.5\%$ for the applied correction (Fig.~\ref{fig:downstream}). The number of threshold-selected weak directions can stay fixed while their span rotates, causing the projector to discard a different component. Label-flip frequency therefore understates estimator-level changes.

Using each sequence's published LiDAR-to-body transform ($\|\dvec\|=0.032$--$0.462$~m), direct reassembly agrees with the congruence to $1.7\times10^{-15}$. Under the unnormalized spectrum the degenerate subspace rotates by up to $5.5^{\circ}$ on one sequence and $40.0^{\circ}$ on another, the smallest eigenvalue moves by up to $35\%$, and the correction exceeds one millimetre on $18.5$, $31.3$, $71.2$ and $100\%$ of pairs. Pooled over the $365$ pairs this is $47.9\%$, against the $0.8\%$ on which the binary flag changes, a factor of $58$ at the published extrinsics of these rigs. The most degenerate sequence has an $8.1$~mm median and the largest single difference is $0.85$~m; under the point-displacement metric the corresponding numerical residuals are $2.6\times10^{-6}$ degrees, $8\times10^{-13}$ and $10^{-14}$~m. Table~\ref{tab:instance} resolves one pair by reported quantity; under either metric, the per-axis footprint is the one row that changes.

The sequence whose extrinsic is $0.058$~m exceeds one millimetre on every pair, and its detector flags a degenerate direction on every pair. The one with the largest extrinsic ($0.462$~m) has a median of exactly zero: over half of its pairs are not degenerate, so the projector reduces to the identity in both frames. Yet it produces the largest single outlier. Projector activity governs incidence; the lever arm governs magnitude. Even on the latter sequence the per-axis footprint changes on $47\%$ of pairs under the unnormalized spectrum and $29\%$ under the equivariant metric. Differences are reported in absolute terms because the identity-initialized one-shot registration itself misses the corridor-axis motion: the median ground-truth displacement over adjacent frame pairs about one second apart is $0.85$~m against $0.03$~m reported. Its own error, dominated by the degeneracy under study, is therefore no independent accuracy baseline.

\subsection{Extrinsics in Practice}
\label{sec:extrinsics}

Nineteen LiDAR-to-body extrinsics were compiled from primary calibration files, logged static transforms and factory metadata. Their lever arms span $0.032$--$0.552$~m (median $0.114$~m), and eighteen have nonzero screw translation, $\evec^{\top}\dvec\neq0$, so by Proposition~\ref{prop:adapted} admit no adapted constant metric.

The sole zero-pitch case arises from two mounting offsets ($+0.06$ and $-0.06$~m) that cancel in the static transform chain. Its adapted metric satisfies $\Ad(T)^{\top}M_0\Ad(T)=M_0$ to $1.8\times10^{-16}$, but is not block diagonal (off-diagonal norm $3.4$), differs from $I_6$ by relative $1.14$, and is accompanied by a non-orthogonal adjoint with $\|A^{\top}A-I\|/\|I\|=0.36$. No tested detector is invariant even in this exception.

At the largest surveyed lever arm Corollary~\ref{cor:cert} gives a band of $[0.58,1.72]$ for the unnormalized spectrum; each characteristic length has its own band through $\delta_{\mathrm{eff}}$. The certificate applies to the five constant metrics under a non-identity adjoint, $1.06\times10^{6}$ of the $2.51\times10^{6}$ evaluations; the orthogonal control has no adjoint, and the three data-derived metrics have no constant $M_0$. About a tenth of those carry no lever arm. There the band collapses to the point $1$, and Proposition~\ref{prop:ortho} already excludes a change, so these cases are set aside as vacuous. Of the $9.49\times10^{5}$ remaining, $5.97\times10^{5}$ fall outside their band and none changes the degeneracy decision, against $21.0\%$ among the $3.52\times10^{5}$ inside it.

\noindent\textit{Scope.} The four baselines are implemented through their implied metrics; one is also reproduced from its published description. For the remaining three, the corresponding publications are the only implementation references. Threshold calibration uses simplified surface sampling with a single spacing per configuration rather than a specific scan pattern. At the coarsest voxel, $0.20$~m against a normal-estimation radius of $0.30$~m, the median neighbourhood is seven points, below the required eight. At the highest noise level, this is where the unnormalized spectrum and $\ell=10$ lose the threshold they hold elsewhere. Section~\ref{sec:downstream} measures one projection step, not accumulated trajectory error.

\section{Conclusion}

A body-frame change sends the point-to-plane information matrix through a congruence. Nullity and the adjoint-identified physical twist subspace are preserved; the per-axis footprint and a thresholded spectrum are not. The circular tunnel shows this without noise or threshold ambiguity.

Universal spectral invariance over positive-semidefinite information forms holds exactly when the metric rule is equivariant. Fixed metrics fail this universal condition. A rig-adapted constant metric exists only at zero screw pitch, a condition observed once in nineteen calibrations; the corresponding metric has a form unlike those in use. The point-displacement metric $M=\sum_iJ_i^{\top}J_i$ satisfies the condition and yields dimensionless, scene-scale-invariant eigenvalues. Its generalized eigenvalues are invariant across frames, consistent unit changes, and scene scales; under the tested fixed-parameter coordinate rescaling and in the cross-sequence comparison, its threshold preserves every decision without retuning a dimensioned normalization parameter. For fixed metrics, the effective lever arm bounds the drift and certifies when no frame change can cross the threshold. Yet per-axis footprints change under any metric because six independent bits cannot encode a screw. Degeneracy should therefore be reported as a subspace of $\se$ with its dimensionless eigenvalues.

\bibliographystyle{IEEEtran}
\bibliography{references}

@inproceedings{Zhang2016,
  author    = {Zhang, Ji and Kaess, Michael and Singh, Sanjiv},
  title     = {On Degeneracy of Optimization-Based State Estimation Problems},
  booktitle = {Proc. IEEE Int. Conf. Robot. Autom. (ICRA)},
  year      = {2016},
  pages     = {809--816},
  doi       = {10.1109/ICRA.2016.7487211}
}

@inproceedings{Hinduja2019,
  author    = {Hinduja, Akshay and Ho, Bing-Jui and Kaess, Michael},
  title     = {Degeneracy-Aware Factors with Applications to Underwater {SLAM}},
  booktitle = {Proc. IEEE/RSJ Int. Conf. Intell. Robots Syst. (IROS)},
  year      = {2019},
  pages     = {1293--1299},
  doi       = {10.1109/IROS40897.2019.8968577}
}

@inproceedings{Nubert2022,
  author    = {Nubert, Julian and Walther, Etienne and Khattak, Shehryar and Hutter, Marco},
  title     = {Learning-Based Localizability Estimation for Robust {LiDAR} Localization},
  booktitle = {Proc. IEEE/RSJ Int. Conf. Intell. Robots Syst. (IROS)},
  year      = {2022},
  pages     = {17--24},
  doi       = {10.1109/IROS47612.2022.9982257}
}

@article{XICP2024,
  author  = {Tuna, Turcan and Nubert, Julian and Nava, Yoshua and Khattak, Shehryar and Hutter, Marco},
  title   = {{X-ICP}: Localizability-Aware {LiDAR} Registration for Robust Localization in Extreme Environments},
  journal = {IEEE Trans. Robot.},
  year    = {2024},
  volume  = {40},
  pages   = {452--471},
  doi     = {10.1109/TRO.2023.3335691}
}

@article{DRPM2024,
  author  = {Hatleskog, Johan and Alexis, Kostas},
  title   = {Probabilistic Degeneracy Detection for Point-to-Plane Error Minimization},
  journal = {IEEE Robot. Autom. Lett.},
  year    = {2024},
  volume  = {9},
  number  = {12},
  pages   = {11234--11241},
  doi     = {10.1109/LRA.2024.3484153}
}

@article{FieldAnalysis2025,
  author  = {Tuna, Turcan and Nubert, Julian and Pfreundschuh, Patrick and Cadena, C{\'e}sar and Khattak, Shehryar and Hutter, Marco},
  title   = {Informed, Constrained, Aligned: A Field Analysis on Degeneracy-Aware Point Cloud Registration in the Wild},
  journal = {IEEE Trans. Field Robot.},
  year    = {2025},
  volume  = {2},
  pages   = {485--515},
  doi     = {10.1109/TFR.2025.3576053}
}

@inproceedings{SuperLoc2025,
  author    = {Zhao, Shibo and Zhu, Honghao and Gao, Yuanjun and Kim, Beomsoo and Qiu, Yuheng and Johnson, Aaron M. and Scherer, Sebastian A.},
  title     = {{SuperLoc}: The Key to Robust {LiDAR}-Inertial Localization Lies in Predicting Alignment Risks},
  booktitle = {Proc. IEEE Int. Conf. Robot. Autom. (ICRA)},
  year      = {2025},
  pages     = {14080--14086},
  doi       = {10.1109/ICRA55743.2025.11127720}
}

@article{DCReg2026,
  author        = {Hu, Xiangcheng and Chen, Xieyuanli and Jia, Mingkai and Wu, Jin and Tan, Ping and Waslander, Steven L.},
  title         = {{DCReg}: Decoupled Characterization for Efficient Degenerate {LiDAR} Registration},
  journal       = {Int. J. Robot. Res.},
  year          = {2026},
  note          = {to appear},
  eprint        = {2509.06285},
  archivePrefix = {arXiv},
  primaryClass  = {cs.RO}
}

@article{Papais2026,
  author        = {Papais, Katya M. and Zhao, Wenda and Barfoot, Timothy D.},
  title         = {Degeneracy-Resilient Teach and Repeat for Geometrically Challenging Environments Using {FMCW} {Lidar}},
  journal       = {arXiv preprint arXiv:2603.10248},
  year          = {2026},
  eprint        = {2603.10248},
  archivePrefix = {arXiv},
  primaryClass  = {cs.RO},
  doi           = {10.48550/arXiv.2603.10248}
}

@article{GEODE2026,
  author  = {Chen, Zhiqiang and Qi, Yuhua and Feng, Dapeng and Zhuang, Xuebin and Chen, Hongbo and Hu, Xiangcheng and Wu, Jin and Peng, Kelin and Lu, Peng},
  title   = {Heterogeneous {LiDAR} Dataset for Benchmarking Robust Localization in Diverse Degenerate Scenarios},
  journal = {Int. J. Robot. Res.},
  year    = {2026},
  volume  = {45},
  number  = {1},
  pages   = {6--22},
  doi     = {10.1177/02783649251344967}
}

@article{Wang2026,
  author  = {Wang, Jiancheng and Cai, Chenyuan and Wang, Yifei and Li, Yuxiang and Zhang, Shiwu and Chen, Haoyao},
  title   = {Screw-Based Feature Constraint Model and Degeneracy Analysis for Robotic State Estimation: Theory and Experiments},
  journal = {Int. J. Robot. Res.},
  year    = {2026},
  note    = {{OnlineFirst}, doi: 10.1177/02783649261463720},
  doi     = {10.1177/02783649261463720}
}

@book{Barfoot2024,
  author    = {Barfoot, Timothy D.},
  title     = {State Estimation for Robotics},
  publisher = {Cambridge Univ. Press},
  address   = {Cambridge, U.K.},
  edition   = {2nd},
  year      = {2024},
  doi       = {10.1017/9781009299909}
}

@article{Sola2018,
  author        = {Sol{\`a}, Joan and Deray, J{\'e}r{\'e}mie and Atchuthan, Dinesh},
  title         = {A Micro Lie Theory for State Estimation in Robotics},
  journal       = {arXiv preprint arXiv:1812.01537},
  year          = {2018},
  eprint        = {1812.01537},
  archivePrefix = {arXiv},
  primaryClass  = {cs.RO},
  doi           = {10.48550/arXiv.1812.01537}
}

@inproceedings{Gelfand2003,
  author    = {Gelfand, Natasha and Ikemoto, Leslie and Rusinkiewicz, Szymon and Levoy, Marc},
  title     = {Geometrically Stable Sampling for the {ICP} Algorithm},
  booktitle = {Proc. 4th Int. Conf. 3-D Digit. Imag. Model. (3DIM)},
  year      = {2003},
  pages     = {260--267},
  doi       = {10.1109/IM.2003.1240258}
}

@inproceedings{Censi2007,
  author    = {Censi, Andrea},
  title     = {An Accurate Closed-Form Estimate of {ICP}'s Covariance},
  booktitle = {Proc. IEEE Int. Conf. Robot. Autom. (ICRA)},
  year      = {2007},
  pages     = {3167--3172},
  doi       = {10.1109/ROBOT.2007.363961}
}

@inproceedings{Nobili2018,
  author    = {Nobili, Simona and Tinchev, Georgi and Fallon, Maurice},
  title     = {Predicting Alignment Risk to Prevent Localization Failure},
  booktitle = {Proc. IEEE Int. Conf. Robot. Autom. (ICRA)},
  year      = {2018},
  pages     = {1003--1010},
  doi       = {10.1109/ICRA.2018.8462890}
}

@inproceedings{Zhen2019,
  author    = {Zhen, Weikun and Scherer, Sebastian},
  title     = {Estimating the Localizability in Tunnel-Like Environments Using {LiDAR} and {UWB}},
  booktitle = {Proc. IEEE Int. Conf. Robot. Autom. (ICRA)},
  year      = {2019},
  pages     = {4903--4908},
  doi       = {10.1109/ICRA.2019.8794167}
}

@inproceedings{Nissov2024,
  author    = {Nissov, Morten and Khedekar, Nikhil and Alexis, Kostas},
  title     = {Degradation Resilient {LiDAR}-Radar-Inertial Odometry},
  booktitle = {Proc. IEEE Int. Conf. Robot. Autom. (ICRA)},
  year      = {2024},
  pages     = {8587--8594},
  doi       = {10.1109/ICRA57147.2024.10611444}
}

@inproceedings{Vaidis2024,
  author    = {Vaidis, Maxime and Hassanzadeh Shahraji, Mohsen and Daum, Effie and Dubois, William and Gigu{\`e}re, Philippe and Pomerleau, Fran{\c{c}}ois},
  title     = {{RTS-GT}: Robotic Total Stations Ground Truthing Dataset},
  booktitle = {Proc. IEEE Int. Conf. Robot. Autom. (ICRA)},
  year      = {2024},
  pages     = {17050--17056},
  doi       = {10.1109/ICRA57147.2024.10610998}
}

@inproceedings{ZhangM3DGR2025,
  author    = {Zhang, Deteng and Zhang, Junjie and Sun, Yan and Li, Tao and Yin, Hao and Xie, Hongzhao and Yin, Jie},
  title     = {Towards Robust Sensor-Fusion Ground {SLAM}: A Comprehensive Benchmark and a Resilient Framework},
  booktitle = {Proc. IEEE/RSJ Int. Conf. Intell. Robots Syst. (IROS)},
  year      = {2025},
  pages     = {8894--8901},
  doi       = {10.1109/IROS60139.2025.11247507}
}

@article{Ostrowski1959,
  author  = {Ostrowski, Alexander M.},
  title   = {A Quantitative Formulation of {S}ylvester's Law of Inertia},
  journal = {Proc. Natl. Acad. Sci. USA},
  year    = {1959},
  volume  = {45},
  number  = {5},
  pages   = {740--744},
  doi     = {10.1073/pnas.45.5.740}
}

\end{document}